\UseRawInputEncoding 
\documentclass[3p,times, review]{elsarticle}

\usepackage{amssymb}
\usepackage{epstopdf}
\usepackage{amsthm}
\usepackage[fleqn]{amsmath}
\usepackage{graphics}
\usepackage{subfigure}
\usepackage[figuresright]{rotating}
\usepackage{booktabs}
\usepackage{CJK}

\newtheorem{thm}{Theorem}

\newtheorem{pro}{Proposition}
\newtheorem{cor}{Corollary}
\newdefinition{rmk}{Remark}
\newdefinition{df}{Definition}
\newdefinition{ex}{Example}
\newproof{pf}{Proof}

 \biboptions{comma,sort&compress}

\begin{document}

\begin{frontmatter}




\title{Comparison research on binary relations based on transitive degrees and cluster degrees}


\author{Zhaohao Wang\corref{cor1}}\ead{nysywzh@163.com} \author{Huifang Yue}
\address{School of Mathematics and Computer Science, Shanxi Normal University,
           Shanxi, Linfen, China} \cortext[cor1]{Corresponding author}

\begin{abstract}
Interval-valued information systems are generalized models of single-valued information systems. By rough set approach, interval-valued information systems have been extensively studied. Authors could establish many binary relations from the same interval-valued information system. In this paper, we do some researches on comparing these binary relations so as to provide numerical scales for choosing suitable relations in dealing with interval-valued information systems. Firstly, based on similarity degrees, we compare the most common three binary relations induced from the same interval-valued information system. Secondly, we propose the concepts of transitive degree and cluster degree, and investigate their properties. Finally, we provide some methods to compare binary relations by means of the transitive degree and the cluster degree. Furthermore, we use these methods to analyze the most common three relations induced from Face Recognition Dataset, and obtain that $RF_{B} ^{\lambda}$ is a good choice when we deal with an interval-valued information system by means of rough set approach.
\end{abstract}


\end{frontmatter}



\section{Introduction}
Rough set theory \cite{16} is an effective mathematical tool for dealing with inaccurate, fuzzy and uncertain data, and it has been successfully applied in many fields, such as machine learning, pattern recognition, financial analysis, decision analysis and etc.\cite{17,18, 19, 27, 28, 29, 30, 31}.

Initially, authors use rough set theory to explore information systems with discrete values. However, there are many forms of data in practical applications, such as set-valued data \cite{37, 38, 40, 41}, interval-valued data \cite{42, 43, 44, 45}, incomplete data \cite{25, 32, 33, 34, 39}, gray data \cite{35, 36} and etc.. Interval values are an extremely common form in information systems. In recent years, by means of generalized rough set model, the interval-valued information system has been widely studied by scholars \cite{4,5,11,13,15,21,22,26}. Dai et al. \cite{5,11} obtained the similarity relation from interval-valued information systems, and then they presented $\theta$-similarity entropy and conditional entropy to deal with the uncertainty measurement problem in interval-valued information systems.
Zhang et al. \cite{26} introduced the $\alpha$-dominance relation from the interval-valued information system so as to investigate this system. Du et al. \cite{23} investigated another kind of dominance relation based on partial order and its corresponding attribute reduction problem. Based on the $\alpha$-dominance relation, Yang et al. \cite{15} proposed a parameterized dominance-based rough set model in the interval-valued information system. Based on $\alpha$-weak similarity relation, Dai et al. \cite{4} studied uncertainty measurement in incomplete interval-valued information systems.

From these references, we can see that an available binary relation plays a key role in investigating interval-valued information systems. However, these studies obtained many binary relations from the same interval-valued information system.
Which is the binary relation better than the others in handling an interval-valued information system? It is interesting to solve this problem. This paper is devoted to study this problem. Since the binary relations established by authors do not usually satisfy the transitivity, we propose the concept of transitive degree to compare the transitive ability of the binary relations by means of transitive closure. In addition, considering the distribution of the values in information systems, we give the concept of cluster degree to analyze the binary relations. Then we establish some methods to distinguish the binary relations by means of the transitive degree and the cluster degree. Without loss of generality, by using these methods, we compare the most common three binary relations induced from the same interval-valued information system. This work provides several numerical scales for choosing suitable relations in dealing with interval-valued information systems.

The rest of this paper is organized as follows. Some of the related concepts involved in this paper are reviewed in the preliminary of Section 2. Section 3 discusses the connection between the binary relations $RF_{B} ^{\lambda}, RS_{B} ^{\lambda}$ and $RT_{B} ^{\lambda}$ from the viewpoint of similarity degree. Section 4 defines the concepts of transitive degrees and cluster degrees, and investigates their properties. Section 5 gives some ways to analyze binary relations from the viewpoint of the transitive degree. The reduction based on the transitive degree is established, and is used to compare binary relations. Section 6 provides some methods to analyze binary relations from the viewpoint of the cluster degree. The reduction based on the cluster degree is established, and is used to compare binary relations. Section 7 concludes this paper.

\section{Preliminaries}
In this section, we review some basic notions related to rough set theory. In addition, some concepts related to interval-valued information systems are introduced.

\subsection{ Basic concepts in rough set theory}
A finite nonempty set $U$ is called a universe. $R\subseteq U \times U$ is referred as to a binary relation on $U$. If for any $x \in U$, $(x, x)\in R$, then $R$ is called reflexive. If for any $x, y\in U$, $(x, y)\in R$ implies that $(y, x)\in R$, then $R$ is called symmetric. If for any $x, y, z \in U$, $(x, y)\in R$ and $(y, z)\in R$ imply that $(x, z)\in R$, then $R$ is called transitive.

If $R$ is reflexive, symmetric and transitive, then $R$ is referred to as an equivalence relation. If $R$ is reflexive and symmetric then $R$ is referred to as a similarity relation.

Assume that $R$ is a binary relation on $U$. For any $x \in U$, the successor neighborhood $R(x)$ of $x$ is defined by
\begin{equation}\label{eq1}
R(x) = \{y \in U\mid (x, y) \in R\}.
\end{equation}
When $R$ is an equivalence relation, $R(x)$ is usually called an equivalence class, and denoted by $[x]_{R}$.

If $R_{1}$ and $R_{2}$ are binary relations on $U$, then the following property holds:
\begin{equation}\label{rpr}
R_{1} \subseteq R_{2} \Leftrightarrow ~\forall x \in U,~R_{1}(x) \subseteq R_{2}(x).
\end{equation}
If $R_{1} \subseteq R_{2}$, then we call $R_{1}$ is finer than $R_{2}$, or $R_{2}$ is coarser than $R_{1}$. We also say that the fine degree of $R_{1}$ is greater than that of $R_{2}$, or the coarse degree of $R_{2}$ is greater than that of $R_{1}$.

Pawlak \cite{16} presented the concepts of lower and upper approximations. Then the definitions of lower and upper approximations were generalized the following version.

\begin{df} \label{app} Let $R$ be a binary relation on $U$. For any $X \subseteq U$, the lower and the upper approximations of $X$ with respect to $R$ are defined as follows:

$\underline{apr}_{R}(X) = \{x\in U \mid R(x)\subseteq X\}$ and $\overline{apr}_{R}(X) = \{x\in U \mid R(x)\cap X \neq \emptyset\}.$
\end{df}

Pawlak \cite{16} proposed two numerical measures for evaluating the uncertainty of rough sets: accuracy and roughness. Dai \cite{11} used these measures to evaluate the uncertainty, accuracy and roughness in the interval-valued information systems. Assume that $R$ is a reflexive relation on $U$ and $X \subseteq U$. The accuracy and roughness of $X$ are given as follows:
\begin{equation}\label{le6}
\alpha_{R}(X)=\frac{|\underline{apr}_{R}(X)|}{|\overline{apr}_{R}(X)|},
\end{equation}
\begin{equation}\label{le7}
\rho_{R}(X)=1-\alpha_{R}(X)=1-\frac{|\underline{apr}_{R}(X)|}{|\overline{apr}_{R}(X)|},
\end{equation}
where $| \cdot |$ denotes the cardinality of a set.
\subsection{Interval-valued information systems}
In this section, we introduce some basic concepts about interval-valued information systems. Firstly, we review the notion of interval values \cite{2}.

In fact, an interval value is a closed interval which is denoted by $u = [u^{-}, u^{+}]$, where $u^{-}, u^{+}\in\mathbb{R}$ and $u^{-}\leq u^{+}$, where $\mathbb{R}$ is the set of real numbers. If $u^{-} = u^{+}$, the interval $u$ will degenerate into a single real number.

Let $u$ and $v$ be interval values, that is, $u = [u^{-}, u^{+}]$  and  $v = [v^{-}, v^{+}]$. The the intersection and union of $u$ and $v$ are defined as follows:
\begin{equation}\label{capi}
u \cap v =
 \left\{\begin{array}{l}
\displaystyle {[\max \{u^{-}, v^{-}\}, \min \{u^{+}, v^{+}\}],~~\max \{u^{-}, v^{-}\} \leq \min \{u^{+}, v^{+}\}},\\
\displaystyle {\emptyset,~~~~~~~~~~~~~~~~~~~~~~~~~~~~~~~~~~~~~~~~~~Otherwise,}\\
\end {array}\right.
\end{equation}
\begin{equation}\label{cupi}
u \cup v = [\min \{u^{-}, v^{-}\}, \max \{u^{+}, v^{+}\}].
\end{equation}

{\begin{df}\label{inin}\cite{3} An interval-valued information system is a quadruple $(U,A,V,f )$, where $U$ is a finite nonempty set called the universe, $A$ is a finite nonempty set of condition attributes, $V = \cup_{a \in A} V_{_{a}}$, where $V_{a}$ is called the domain of attribute $a$, and $f : U \times A \longrightarrow V$ is called a total function, such that $f(x,a) = [f(x,a)^{-}, f(x,a)^{+}] \in V_{a}$ is an interval value for every $a \in A$ and $x \in U$.
\end{df}

\subsection{Similarity relations for interval-valued information systems}\label{sec23}
Nakahara et al. \cite{5} proposed the concept of similarity degree which is used to compare interval values. By means of this concept, authors \cite{3,4,5,6,7,8} investigated the interval-valued information systems. They used similarity degrees to induce binary relations, and then constructed rough set models in interval-valued information systems. By rough set approach, they analyzed interval-valued information systems. In this section, we review three similarity degrees, and we introduce three most common similarity relations.

\begin{rmk}\label{rm}
Let $S(u,v)$ be a function from $V\times V$ to $\mathbb{R}$. The following properties are commonly used to characterize a similarity degree \cite{li,he}:

(a) $0\leq S(u,v)\leq 1$.

(b) $S(u,v)=1$ if and only if $u=v$.

(c) $S(u,v)=S(v,u)$.

\noindent That is to say, if $S(u,v)$ satisfies the properties (a)-(c), then $S(u,v)$ is a similarity degree.
\end{rmk}

Let $V$ be the universe of interval values. For any $u, v\in V$, the most common three similarity degrees of $u$ and $v$ are given as follows:
\begin{equation}\label{rf}
SF(u, v) = \frac{|u\cap v|}{|u\cup v|},
\end{equation}
where $|\cdot |$ denotes the length of closed interval and the length of an empty interval or a single point is zero.
\begin{equation}\label{rs}
SS(u, v) = 1 - \frac{1}{2} \ast \frac{|u^{+} -  v^{+}| + |u^{-} -  v^{-}|}{\max\{u^{+}, v^{+}\} - \min\{u^{-}, v^{-}\}},
\end{equation}
\begin{equation}\label{rt}
ST(u, v) = 1 - |P(u \geq v) - P(v\geq u)|,
\end{equation}
where $P(u\geq v) = \min \Big\{1, \max \{\frac{u^{+} -  v^{-}}{|u| + |v|}, 0\}\Big\}.$

Let $(U, A, V, f)$ be an interval-valued information system. By means of the three similarity degrees, for $B \subseteq A$ and $\lambda \in [0, 1]$, three $\lambda$-similarity relations with respect to $B$ can be constructed as follows:
\begin{equation}\label{fr}
RF_{B} ^{\lambda}=\{(x, y)\in U\times U\mid SF(f(x, a), f(y, a))\geq\lambda, \forall a \in B\},
\end{equation}
\begin{equation}\label{sr}
RS_{B}^{\lambda}=\{(x, y) \in U\times U\mid SS(f(x, a), f(y, a))\geq\lambda, \forall a \in B \},
\end{equation}
\begin{equation}\label{tr}
RT_{B}^{\lambda} = \{(x, y) \in U\times U\mid ST(f(x, a), f(y, a))\geq\lambda, \forall a \in B\}.
\end{equation}

\subsection{Transitive closure}
In this section, we introduce the concept of transitive closure of binary relations and the method of calculating transitive closure of a binary relation.
\begin{df}\label{tc} \cite{12} Let $R$ and $t(R)$ be general binary relations on $U$. $t(R)$ is called the transitive closure of $R$, if the following conditions hold:

(T1) $t(R)$ is transitive.

(T2) $R \subseteq t(R)$.

(T3) For any transitive relation $R^{\prime}$ on $U$, if $R \subseteq R^{\prime}$, then $t(R)\subseteq R^{\prime}$.
\end{df}

From the above definition, we can see that the transitive closure $t(R)$ of $R$ is the minimum transitive relation containing $R$.

Let $U=\{x_{1},x_{2},\cdots,x_{n}\}$ and $R$ be a binary relation on $U$. In this paper, we denote the relation matrix of $R$ as $M_{R}=(m_{ij})_{n\times n}$, where
\begin{equation}
m_{ij}=\left\{
  \begin{array}{ll}
    1, & \hbox{$(x_{i},x_{j})\in R$,} \\
    0, & \hbox{$(x_{i},x_{j})\notin R$.}
  \end{array}
\right.
\end{equation}
For example, let $U = \{x_{1}, x_{2}, x_{3}, x_{4}\}$ and $R = \{ (x_{1}, x_{1}), (x_{1}, x_{2}), (x_{2}, x_{2}), (x_{2}, x_{3}), (x_{3}, x_{3} ), (x_{4}, x_{4})\}$. Then the relation matrix $M_{R}$ of $R$ can be computed as follows:
$$M_{R} =\left(
  \begin{array}{ccccc}
    1 & 1 & 0 & 0 \\
    0 & 1 & 1 & 0 \\
    0 & 0 & 1 & 0 \\
    0 & 0 & 0 & 1 \\
  \end{array}
\right).$$
Next, we will give a way to calculate the transitive closure.
\begin{pro}\label{tcp} \cite{12} Let $R$ be a reflexive binary relation on $U$ and $t(R)$ be the transitive closure of $R$. Then there exists $l\in\{1,2,\cdots,|U|\}$ such that
\begin{center}
$M _{t(R)} = M_{R}^{l},$
\end{center}
where $M_{R}^{l} = \underbrace{M_{R} \circ M_{R} \circ, \cdots, \circ M_{R}}\limits_{l}$,  $M_{R}^{l} = M_{R}^{l + 1}$, and $\circ$ represents the Boolean operation.
\end{pro}
Clearly, according to the relation matrix $M _{t(R)}$, we can obtain the binary relation $t(R)$. That is to say, we can compute the transitive closure $t(R)$ of $R$ according to Proposition \ref{tcp}.

\section{Comparison research on binary relations, $RF_{B} ^{\lambda}$, $RS_{B} ^{\lambda}$ and $RT_{B} ^{\lambda}$, based on similarity degrees}
The $\lambda$-similarity relations $RF_{B} ^{\lambda},$ $RS_{B} ^{\lambda}$ and $RT_{B} ^{\lambda}$ have been studied by some scholars, but there is little research on the relationship between them. In this section, we establish their connections.

These $\lambda$-similarity relations are constructed by the corresponding similarity degrees. Therefore we first give the relationship among the similarity degrees $SF(u, v)$, $SS(u, v)$ and $ST(u,v)$.
\begin{pro}\label{srpr}
Let $V$ be the universe of interval values. Then $\forall u, v\in V$, $SS(u, v) \geq ST(u, v) \geq SF(u, v)$.
\end{pro}
\begin{proof}
For any $u, v\in V$, by the symmetry of $SS(u, v)$, $SF(u, v)$ and $ST(u, v)$, we only need prove the case of $v^{+} \geq u^{+}$.

If $u\cap v= \emptyset$, by Eqs. (\ref{rf}) and (\ref{rt}), it is easy to compute that $SF(u, v) = ST(u, v) = 0$. It is easy to see that the inequality $SS(u, v) \geq ST(u, v) \geq SF(u, v)$ holds.

If $u\cap v\neq\emptyset$, then by Eq. (\ref{capi}) and $v^{+} \geq u^{+}$, we conclude that $v^{-}< u^{+}$. There are two cases:
\begin{itemize}
  \item Case 1: $u^{-}\leq v^{-}< u^{+}$. Now, according to Eqs. (\ref{capi}) and (\ref{cupi}), we have $u\cap v=[v^{-},u^{+}]$ and $u\cup v=[u^{-},v^{+}]$. Thus, by Eqs. (\ref{rf}), (\ref{rs}) and (\ref{rt}), we obtain that
\begin{align}
&\label{rf1} SF(u, v)=\frac{|u\cap v|}{|u\cup v|}=\frac{u^{+}-v^{-}}{v^{+}-u^{-}};\\
&\label{rs1} SS(u, v)=1-\frac{1}{2}\ast\frac{v^{+} -u^{+}+v^{-} - u^{-}}{v^{+} - u^{-}}=\frac{u^{+} - u^{-}+v^{+}-v^{-}}{2(v^{+} - u^{-})}=\frac{|u|+|v|}{2(v^{+} - u^{-})};\\
&\label{rt1} ST(u, v)=1-|P(u\geq v)-P(v \geq u)|=1-\frac{v^{+}+v^{-}-u^{-}-u^{+}}{|u|+|v|} =
\frac{2(u^{+}-v^{-})}{|v| + |u|}.
\end{align}
Clearly, $|u|\leq |u\cup v|$ and $|v|\leq |u\cup v|$. Thus,
\begin{equation}\label{ieq2}
|u|+|v|\leq 2|u\cup v|=2(v^{+}-u^{-}).
\end{equation}
By Eqs. (\ref{rf1}), (\ref{rt1}) and (\ref{ieq2}), we conclude that
\begin{equation}\label{tf}
ST(u, v)=\frac{2(u^{+}-v^{-})}{|v| + |u|}\geq\frac{2(u^{+}-v^{-})}{2(v^{+} - u^{-})}=\frac{u^{+}-v^{-}}{v^{+} - u^{-}}=SF(u, v).
\end{equation}
In addition, according to Eqs. (\ref{rs1}) and (\ref{rt1}), we can obtain that
\begin{align}
&\label{rs2} SS(u, v)=\frac{|u|+|v|}{2(v^{+} - u^{-})}=\frac{(u^{+}-v^{-})+(v^{+}-u^{-})}{(v^{+} - u^{-})+(v^{+} - u^{-})};\\
&\label{rt2} ST(u, v)=\frac{2(u^{+}-v^{-})}{|v| + |u|}=\frac{(u^{+}-v^{-})+(u^{+}-v^{-})}{(v^{+} - u^{-})+(u^{+} - v^{-})}.
\end{align}
It is clear that $v^{+}-u^{-}=|u\cup v|\geq|u\cap v|=u^{+}-v^{-}$. For simplicity, we write $a=v^{+}-u^{-}$ and $b=u^{+}-v^{-}$. That is, $a\geq b$. And, by Eqs. (\ref{rs2}) and (\ref{rt2}), we can conclude that
\begin{equation}\label{st}
SS(u, v)=\frac{(u^{+}-v^{-})+a}{(v^{+} - u^{-})+a}\geq \frac{(u^{+}-v^{-})+b}{(v^{+} - u^{-})+b}=ST(u, v).
\end{equation}
Combining Eq. (\ref{tf}) and Eq. (\ref{st}), we know that the inequality $SS(u, v) \geq ST(u, v) \geq SF(u, v)$ holds.

  \item Case 2: $v^{-}< u^{-}\leq u^{+}$. Now, according to Eqs. (\ref{capi}) and (\ref{cupi}), we have $u\cap v=u=[u^{-},u^{+}]$ and $u\cup v=v=[v^{-},v^{+}]$. Thus, by Eqs. (\ref{rf}), (\ref{rs}) and (\ref{rt}), we obtain that
\begin{align}
&\label{rf3} SF(u, v)=\frac{u^{+}-u^{-}}{v^{+}-v^{-}};\\
&\label{rs3} SS(u, v)=\frac{|u|+|v|}{2(v^{+} - v^{-})}=\frac{|u|+|v|}{2|v|};\\
&\label{rt3} ST(u, v)=\frac{2(u^{+}-v^{-})}{|v| + |u|}.
\end{align}
Clearly, $|u|\leq |u\cup v|$ and $|v|\leq |u\cup v|$. Thus,
\begin{equation}\label{ieq3}
|u|+|v|\leq 2|u\cup v|=2(v^{+}-v^{-}).
\end{equation}
By Eqs. (\ref{rf3}), (\ref{rt3}) and (\ref{ieq3}), and $v^{-}< u^{-}$, we conclude that
\begin{equation}\label{tf1}
ST(u, v)=\frac{2(u^{+}-v^{-})}{|v| + |u|}\geq\frac{2(u^{+}-v^{-})}{2(v^{+} - v^{-})}\geq\frac{u^{+}-u^{-}}{v^{+} - v^{-}}=RF(u, v).
\end{equation}
In addition, according to Eqs. (\ref{rs3}) and (\ref{rt3}) and $u^{-}> v^{-}$, we can obtain that
\begin{align*}
SS(u, v)-ST(u, v)&=\frac{(|u|+|v|)^{2}-4|v|(u^{+}-v^{-})}{2|v|(|v| + |u|)}\\
&\geq \frac{(|u|+|v|)^{2}-4|v|(u^{+}-u^{-})}{2|v|(|v| + |u|)}=\frac{(|u|+|v|)^{2}-4|v||u|}{2|v|(|v| + |u|)}=\frac{(|u|-|v|)^{2}}{2|v|(|v| + |u|)}\geq 0.
\end{align*}
Thus, we have $SS(u, v)\geq ST(u, v)$. According to Eq. (\ref{tf1}), we get that the inequality $SS(u, v) \geq ST(u, v) \geq SF(u, v)$ holds.
\end{itemize}

In summary, $SS(u, v) \geq ST(u, v) \geq SF(u, v)$ holds. This completes the proof.
\end{proof}

\begin{thm}\label{thm1} Let $(U, A, V, f)$ be an interval-valued information system and $B \subseteq A$. Then $\forall \lambda \in [0, 1]$,
\begin{equation*}
RF_{B} ^{\lambda} \subseteq RT_{B} ^{\lambda}\subseteq RS_{B} ^{\lambda}.
\end{equation*}
\end{thm}
\begin{proof}
It is straightforward from Proposition \ref{srpr} and Eqs. (\ref{fr}), (\ref{sr}) and (\ref{tr}).
\end{proof}

By Theorem \ref{thm1} and Eqs. (\ref{le6}) and (\ref{le7}), we can conclude the following result.

\begin{pro}\label{pr2} Let $(U, A, V, f)$ be an interval-valued information system and $B \subseteq A$. Then the following statements hold: for $X\subseteq U$,

(1) $\alpha_{RS_{B} ^{\lambda}}(X) \leq \alpha_{RT_{B} ^{\lambda}}(X) \leq \alpha_{RF_{B} ^{\lambda}}(X)$.

(2) $\rho_{RF_{B} ^{\lambda}}(X) \leq \rho_{RT_{B} ^{\lambda}}(X) \leq \rho_{RS_{B} ^{\lambda}}(X)$.
\end{pro}

\begin{table}[htbp]
\centering
\caption{\label{tb1} An interval-valued information system}
\begin{tabular}{cccccc}
\toprule
& $a_{1}$& $a_{2}$& $a_{3}$& $a_{4}$\\
\midrule
$x_{1}$ & $[0, 0.5]$ & $[0.2, 0.7]$ & $[0.3, 0.6]$ & $[0.1, 0.3]$\\
$x_{2}$ & $[0.2, 0.6]$ & $[0, 0.5]$ & $[0.1, 0.7]$ & $[0.3, 0.7]$\\
$x_{3}$ & $[0.1, 0.8]$ & $[0.3, 0.8]$ & $[0, 0.8]$ & $[0.5, 0.9]$\\
$x_{4}$ & $[0, 0.9]$ & $[0.4, 1]$ & $[0.2, 0.6]$ & $[0.6, 1]$\\
$x_{5}$ & $[0.6, 1]$ & $[0.1, 0.5]$ & $[0.3, 0.9]$ & $[0.8, 1]$\\
\bottomrule
\end{tabular}
\end{table}

\begin{ex}\label{ex1} For the interval-valued information system shown in Table \ref{tb1}, we take $B = \{a_{1}, a_{2}, a_{3}\}$ and $\lambda = 0.6$. By Eq. (\ref{fr}), we can obtain that
$$RF_{B} ^{0.6}(x_{1}) = \{x_{1}\}, ~~RF_{B} ^{0.6}(x_{2}) = \{x_{2}\}, ~~RF_{B} ^{0.6}(x_{3}) = \{x_{3}\}, ~~RF_{B} ^{0.6}(x_{4}) = \{x_{4}\}, ~~RF_{B} ^{0.6}(x_{5}) = \{x_{5}\}.$$
By Eqs. (\ref{sr}) and (\ref{tr}), we can compute that
$$RS_{B}^{0.6}(x_{1})=\{x_{1}, x_{2}, x_{3}, x_{4}\}, ~~RS_{B} ^{0.6}(x_{2})=\{x_{1}, x_{2}, x_{3}\},
~~RS_{B} ^{0.6}(x_{3})=\{x_{1}, x_{2}, x_{3}, x_{4}, x_{5}\},~~RS_{B} ^{0.6}(x_{4}) = \{x_{1}, x_{3}, x_{4}\},~~RS_{B} ^{0.6}(x_{5})=\{x_{3}, x_{5}\},$$
and
$$RT_{B}^{0.6}(x_{1})=\{x_{1}, x_{2}, x_{3}\},~~RT_{B}^{0.6}(x_{2})=\{x_{1}, x_{2}\}, RT_{B}^{0.6}(x_{3}) =\{x_{1}, x_{3}, x_{4}\},~~RT_{B}^{0.6}(x_{4})=\{x_{3}, x_{4}\},~~RT_{B}^{0.6}(x_{5})=\{x_{5}\}.$$
It is easy to check that $RF_{B} ^{0.6} \subset RT_{B} ^{0.6}\subset RS_{B} ^{0.6}$. This coincides with the result of Theorem \ref{thm1}.

In addition, we choose $X=\{x_{2}, x_{3}\}$. We can compute that
$\alpha_{RF_{B} ^{0.6}}(X) = 1$, $\alpha_{RS_{B} ^{0.6}}(X) = 0$, and $\alpha_{RT_{B} ^{0.6}}(X) = 0.$ Thus, $\alpha_{RS_{B} ^{0.6}}(X) =\alpha_{RT_{B} ^{0.6}}(X) < \alpha_{RF_{B} ^{0.6}}(X)$. This coincides with the result of Proposition \ref{pr2}.
\end{ex}

From the above example, we can see that the relations, $RF_{B} ^{0.6}$, $RS_{B} ^{0.6}$ and $RT_{B} ^{0.6}$, are induced from the same information system. Although $RT_{B} ^{0.6}\subset RS_{B} ^{0.6}$, $\alpha_{RS_{B} ^{0.6}}(X) =\alpha_{RT_{B} ^{0.6}}(X)$. Therefore the fine degree and accuracy of the relations cannot effectively measure the quality of the relations. In dealing with an interval-valued information system by rough set approach, which should relation be chosen? Just relying on fine degree and accuracy is not enough. Thus we will present several new tools to measure general binary relations in this paper.
\section{Transitive degrees and cluster degrees}
In this section, we propose two new measures for binary relations, transitive degrees and cluster degrees, and discuss their properties.
\subsection{Transitive degrees of binary relations}
When we investigate an information system in terms of rough set theory, we need to construct a binary relation from the information system. In classic rough set model, this binary relation is an equivalence relation. However, the relations induced from interval-valued information systems are usually not equivalence relations. They are usually reflexive and symmetric but not transitive. Thus it is necessary to explore the connection between these binary relations and transitivity. In view of this, authors presented the concept of transitive closure (See Definition \ref{tc}) and investigated its property. In this section, by means of transitive closure, we give the concept of transitive degree of a binary relation to further study the relationship between a general binary relation and transitivity.

\begin{df}\label{trb}
Let $R$ be a reflexive binary relation on $U$. The transitive degree of $R$ is defined as follows:
\begin{equation}
TD_{R} = \frac{1}{|U|}\sum\limits_{x \in U}TD(R(x)),
\end{equation}
where $TD(R(x))=\frac{|R(x)|}{|t(R)(x)|}$ and $t(R)$ is the transitive closure of $R$.
\end{df}

In above definition, $R$ is reflexive, thus $\forall x\in U$, $R(x)\neq\emptyset$, which implies $\forall x\in U$, $t(R)(x)\neq\emptyset$, that is, $|t(R)(x)|\neq 0$. This can ensure that the formula $TD(R(x))=\frac{|R(x)|}{|t(R)(x)|}$ is meaningful.
\begin{pro}\label{pr3}
Let $R$ be a reflexive binary relation on $U$. Then the following statements hold:

(1) $\forall x\in U$, $0<TD(R(x))\leq 1$.

(2) $0 < TD_{R} \leq 1$.
\end{pro}
\begin{proof}
(1) By Definition \ref{tc}, we have $R \subseteq  t(R)$. By Eq. (\ref{rpr}), this implies that $\forall x\in U$, $R(x) \subseteq  t(R)(x)$, and thus $|R(x)|\leq |t(R)(x)|$. Since $R$ is reflexive, it follows that $\forall x\in U$, $|R(x)|>0$. Therefore, $\forall x\in U$, $t(R)(x)\neq 0$. Consequently, $\forall x\in U$, $0<TD(R(x))=\frac{|R(x)|}{|t(R)(x)|}\leq 1$.

(2) By (1), it is obvious.
\end{proof}

\begin{thm}\label{thm2}
Let $R$ be a reflexive relation on $U$. Then $R$ is transitive if and only if $TD_{R}=1$.
\end{thm}
\begin{proof}
Since $R$ is transitive, it follows from Definition \ref{tc} that $R=t(R)$. By Definition \ref{trb}, it is clear that $TD_{R}=1$. This completes the proof of the necessity.

Conversely, since $TD_{R}=1$, it follows from Definition \ref{trb} that $\sum\limits_{x \in U} TD(R(x))=|U|$. By (1) of Proposition \ref{pr3}, we have that $\forall x\in U$, $0<TD(R(x))\leq 1$. This implies that $\forall x\in U$, $TD(R(x))=1$. Therefore, $\forall x\in U$, $|t(R)(x)|=|R(x)|$. It follows from $\forall x\in U$, $R(x)\subseteq t(R)(x)$ that $\forall x\in U$, $t(R)(x)=R(x)$. That is to say, $t(R)=R$. Thus $R$ is transitive.
\end{proof}

\begin{cor}\label{thm2c}
Let $(U, A, V, f)$ be an interval-valued information system, $B\subseteq A$ and $\lambda\in[0,1]$. Then $RF_{B} ^{\lambda}$ (or $RS_{B} ^{\lambda}$, or $RT_{B} ^{\lambda}$) is an equivalence relation if and only if $TD_{RF_{B} ^{\lambda}}=1$ (or $TD_{RS_{B} ^{\lambda}}=1$, or $TD_{RT_{B} ^{\lambda}}=1$).
\end{cor}

Theorem \ref{thm2} shows that $R$ is transitive when the transitive degree is the maximum value 1. This illustrates that the transitive degree given by Definition \ref{trb} is reasonable.

According to Definition \ref{trb}, the following result is obvious.

\begin{pro}\label{pr4}
Let $R_{1}$ and $R_{2}$ be reflexive relations on $U$. If $R_{1} \subseteq R_{2}$ and $t(R_{1}) = t(R_{2})$, then $TD_{R_{1}} \leq TD_{R_{2}}$.
\end{pro}

According to Theorem \ref{thm1}, although $RF_{B} ^{\lambda} \subseteq RT_{B} ^{\lambda}\subseteq RS_{B} ^{\lambda}$ holds, their transitive closures are unequal. Hence $TD_{RF_{B} ^{\lambda}} \leq TD_{RT_{B} ^{\lambda}}\leq TD_{RS_{B} ^{\lambda}}$ is not necessary true. The following example illustrates this point.

\begin{ex}\label{ex2} (Continuation of Example \ref{ex1})  By Proposition \ref{tcp}, we can compute that
\begin{align*}
&t(RF_{B} ^{0.6})(x_{1}) = \{x_{1}\}, ~t(RF_{B} ^{0.6})(x_{2})=\{x_{2}\},~t(RF_{B} ^{0.6})(x_{3}) = \{x_{3}\},~t(RF_{B}^{0.6})(x_{4})=\{x_{4}\},~t(RF_{B}^{0.6})(x_{5}) = \{x_{5}\};\\
&t(RS_{B}^{0.6})(x_{1})=t(RS_{B}^{0.6})(x_{2})=t(RS_{B} ^{0.6})(x_{3})= t(RS_{B}^{0.6})(x_{4}) = t(RS_{B} ^{0.6})(x_{5})=\{x_{1}, x_{2}, x_{3}, x_{4}, x_{5}\};\text{ and }\\
&t(RT_{B} ^{0.6})(x_{1}) = t(RT_{B} ^{0.6})(x_{2}) = t(RT_{B} ^{0.6})(x_{3}) = t(RT_{B} ^{0.6})(x_{4}) = \{x_{1}, x_{2}, x_{3}, x_{4}\},~t(RT_{B} ^{0.6})(x_{5})=\{x_{5}\}.
\end{align*}
Clearly, $t(RF_{B} ^{0.6})$, $t(RS_{B} ^{0.6})$ and $t(RT_{B} ^{0.6})$ are pairwise unequal. In addition, by Definition \ref{trb}, we can compute that
$TD_{RF_{B} ^{0.6}}=1$, $TD_{RS_{B} ^{0.6}}=0.68$ and $TD_{RT_{B} ^{0.6}}=0.7$. It is easy to see that
$$TD_{RS_{B} ^{0.6}} \leq TD_{RT_{B} ^{0.6}} \leq TD_{RF_{B} ^{0.6}}.$$
\end{ex}

\subsection{Cluster degree of binary relations}
In the processing of an interval-valued information system, the first is to classify the domain $U$ by means of a binary relation. However, there are many ways to obtain a binary relation from the same interval-valued information system. That is to say, we can obtain more than one binary relations from an interval-valued information system. In order to distinguish these relations, we need to measure the classification ability of binary relations. Thus, in this subsection, we propose a new measure of binary relations, that is, the cluster degree, to evaluate the classification ability of binary relations.

Average is an important concept in statistics. Since average is the most important measure to describe the central tendency and dispersion degree of dataset, we can use it to define the concept of the cluster degree. First, we give the interval-valued average as follows.

Let $V$ be a set of interval values. We use $\overline{V}$ denote the average value of $V$, then $\overline{V}$ is an interval and we have
\begin{equation}\label{eqa}
\overline{V}^{-}=\frac{\sum_{v\in V}v^{-}}{|V|} ~\text{and}~ \overline{V}^{+}=\frac{\sum_{v\in V}v^{+}}{|V|}.
\end{equation}
That is, $\overline{V}=[\overline{V}^{-},\overline{V}^{+}]$.
For example, set $V = \{ [1, 2], [3, 4], [5, 6]\}$. We can compute $\overline{V}^{-} = \frac{1 + 3 + 5}{|V|} = \frac{1 + 3 + 5}{3} = 3$ and $\overline{V}^{+} = \frac{2 + 4 + 6}{|V|}= \frac{2 + 4 + 6}{3} = 4$. Thus $\overline{V} = [3, 4].$

\begin{rmk}\label{rm2}
In the following section, we write $R_{B}$ as a binary relation induced by attribute subset $B$. If $B$ is a singlet element set, such as $B=\{a\}$, then $R_{B}$ is denoted by $R_{a}$.
\end{rmk}

Let $(U, A, V, f)$ be an interval-valued information system and $B\subseteq A$. For $x\in U$ and $a\in B$, in the following section, we use $V_{R_{B}(x)}^{a}$ to denote the following set of interval values:
\begin{equation}\label{eqa1}
V_{R_{B}(x)}^{a}=\{f(y,a)\mid y\in R_{B}(x)\}.
\end{equation}
Thus, by Eq. (\ref{eqa}), it is easy to check that the average value of $V_{R_{B}(x)}^{a}$ is $\overline{V_{R_{B}(x)}^{a}}=[(\overline{V_{R_{B}(x)}^{a}})^{-},(\overline{V_{R_{B}(x)}^{a}})^{+}]$, where
\begin{equation}\label{eqa2}
(\overline{V_{R_{B}(x)}^{a}})^{-}=\frac{\sum_{y \in R_{B}(x)} f(y, a)^{-}}{|R_{B}(x)|} ~\text{and}~
(\overline{V_{R_{B}(x)}^{a}})^{+}=\frac{\sum_{y \in R_{B}(x)} f(y, a)^{+}}{|R_{B}(x)|}.
\end{equation}

Now, we provide the concept of cluster degree as following.

\begin{df}\label{jl}
Let $(U, A, V, f)$ be an interval-valued information system, $B \subseteq A$ and $R_{B}$ be a reflexive relation induced by $B$. For $x\in U$, the cluster degree of $R_{B}(x)$ is defined as follows:
\begin{equation*}
CD_{R_{B}(x)}=\frac{1}{|B|}\sum\limits_{a \in B} CD_{R_{a}(x)},
\end{equation*}
where $CD_{R_{a}(x)} = \frac{\sum_{y \in R_{B}(x)}S(f(y,a), \overline{V_{R_{B}(x)}^{a}})}{|R_{B}(x)|}$ and $S$ is a similarity degree between interval values.
\end{df}

\begin{rmk}\label{rm1}
In Definition \ref{jl}, the reflexivity of $R_{B}$ can ensure that $\forall a\in B$ and $x\in U$, $R_{a}(x)\neq\emptyset$. That is to say, if $R_{B}$ is reflexive, then $\forall a\in B$ and $x\in U$, $|R_{a}(x)|\neq 0$. Thus, in Definition \ref{jl}, we can take $|R_{a}(x)|$ as a denominator.
\end{rmk}

In Definition \ref{jl}, $R_{B}(x)$ can be seen as a class of successor neighbourhood. It is constructed by attribute values in the interval-valued information system. We know that in the classic case, if two objects have the same values of attributes in an information system, then they are divided into the same class. In this way, the classification results are considered as an exact classification. However, in general, the attribute values of any pair objects are different in interval-valued information systems. Hence authors take usually objects as a class when the attribute values of these objects are little difference. In this way, the classification results should not be an exact classification. Based on this observation, in order to measure the classification results, we give the cluster degree of a class by means of the difference between attribute values and their averages.

Note that the similarity degree is between 0 and 1 (See Remark \ref{rm}), that is, in Definition \ref{jl}, $0\leq S(f(a,y), \overline{V_{R_{B}(x)}^{a}})\leq1$. This implies $0\leq CD_{R_{a}(x)}\leq 1$. Thus we can conclude the following conclusion.

\begin{pro}\label{jlp}
Let $(U, A, V, f)$ be an interval-valued information system and $S$ be a similarity degree between interval values. Then for $B\subseteq A$ and  $x\in U$, $0 \leq CD_{R_{B}(x)}\leq 1$.
\end{pro}

\begin{thm}\label{tej}
Let $(U, A, V, f)$ be an interval-valued information system, $B \subseteq A$ and $R_{B}$ be a reflexive relation induced by $B$. Then for $x\in U$, $CD_{R_{B}(x)} = 1$ if and only if $\forall y,z\in R_{B}(x)$ and $\forall a\in B$, $f(y,a)=f(z,a)$.
\end{thm}
\begin{proof}
By $CD_{R_{B}(x)} = 1$, we conclude that $\forall a \in B$, $CD_{R_{a}(x)} = 1$. By Definition \ref{jl}, this implies $\frac{\sum_{y \in R_{a}(x)}S(f(y,a), \overline{V_{R_{B}(x)}^{a}})}{|R_{a}(x)|}=1$ and thus $\forall y \in R_{a}(x)$, $S(f(y,a), \overline{V_{R_{B}(x)}^{a}})=1$. By Remark \ref{rm}, $f(y,a)=\overline{V_{R_{B}(x)}^{a}}$. We have proved that $\forall y\in R_{B}(x)$, $f(y,a)=\overline{V_{R_{B}(x)}^{a}}$. This implies that $\forall y,z\in R_{B}(x)$ and $\forall a\in B$, $f(y,a)=\overline{V_{R_{B}(x)}^{a}}=f(z,a)$. This completes the proof of the necessity.

Conversely, $\forall y,z\in R_{B}(x)$ and $\forall a\in B$, $f(y,a)=f(z,a)$. It follows from Eq. (\ref{eqa2}) and Remark \ref{rm} that $\forall y\in R_{B}(x)$, $S(f(y,a), \overline{V_{R_{B}(x)}^{a}})=1$. Therefore, by Definition \ref{jl}, we conclude that $CD_{R_{B}(x)} = 1$.
\end{proof}

Next, according to Definition \ref{jl}, we establish the cluster degree of a binary relation.

\begin{df}\label{jlr}
Let $(U, A, V, f)$ be an interval-valued information system, $B \subseteq A$ and $R_{B}$ be a reflexive relation induced by $B$. Then the cluster degree of $R_{B}$ is defined as follows:
$$CD_{R_{B}}=\frac{1}{|U|}\sum_{x\in U}CD_{R_{B}(x)}.$$
\end{df}

\begin{pro}\label{prs}
Let $(U, A, V, f)$ be an interval-valued information system, $B\subseteq A$ and $R_{B}$ be a reflexive relation induced by $B$. Then $0 \leq CD_{R_{B}} \leq 1$.
\end{pro}
\begin{proof}
It is easy to prove according to Proposition \ref{jlp}.
\end{proof}

\begin{thm}\label{thm4}
Let $(U, A, V, f)$ be an interval-valued information system, $B \subseteq A$ and $R_{B}$ be a reflexive relation induced by $B$. Then $CD_{R_{B}}=1$ if and only if $\forall x\in U$ and $y,z\in R_{B}(x)$, $\forall a\in B$, $f(y,a) = f(z,a)$.
\end{thm}
\begin{proof}
It is straightforward from Definition \ref{jlr} and Theorem \ref{tej}.
\end{proof}

Theorem \ref{thm4} provides the condition that the cluster degree reaches the maximum 1. According to Theorem \ref{thm4}, we can obtain the following result.

The connection between $CD_{R_{B}}=1$ and equivalence relations is unclear. This is because in Theorem \ref{thm4}, we only know that $R_{B}$ is a reflexive relation induced by $B$, but the way that $R_{B}$ is induced is unknown. If we know how to generate the relation $R_{B}$, then the following result will become clear.
\begin{cor}\label{cor1}
Let $(U, A, V, f)$ be an interval-valued information system, $B \subseteq A$ and $\lambda\in[0,1]$. Then $CD_{RF_{B} ^{\lambda}}=1$ (or $CD_{RS_{B} ^{\lambda}}=1$, or $CD_{RT_{B} ^{\lambda}}=1$) if and only if $RF_{B} ^{\lambda}$ (or $RS_{B} ^{\lambda}$, or $RT_{B} ^{\lambda}$) is an equivalence relation.
\end{cor}
\begin{proof}
It is straightforward from Theorem \ref{thm4} and Eqs. (\ref{fr})-(\ref{tr}).
\end{proof}

At last, we give an example to illustrate the cluster degree.
\begin{ex}\label{ex3}(Continuation of Example \ref{ex1}) In order to compute $CD_{RF_{B} ^{0.6}}$, we choose the similarity degree $SF$. By Definition \ref{jl}, it is easy to compute that
$$CD_{RF_{a_{1}}^{0.6}(x_{1})}= CD_{RF_{a_{2}}^{0.6}(x_{1})} = CD_{RF_{a_{3}}^{0.6}(x_{1})} = 1,$$ and thus $CD_{RF_{B}^{0.6}(x_{1})} = \frac{1}{3}(1 + 1 + 1) = 1.$ Similarity, we can obtain that $$CD_{RF_{B}^{0.6}(x_{2})} = CD_{RF_{B}^{0.6}(x_{3})} = CD_{RF_{B}^{0.6}(x_{4})} = CD_{RF_{B}^{0.6}(x_{5})} = 1.$$ It follows that $CD_{RF_{B}^{0.6}} = \frac{1}{5} (1 + 1 + 1 + 1 + 1) = 1$.

For $RS_{B} ^{0.6}$, we know that $RS_{B}^{0.6}(x_{1})= \{x_{1}, x_{2}, x_{3}, x_{4}\}$. According to Eq. (\ref{eqa2}), we have
\begin{align*}
&\overline{V}_{RS_{a_{1}}^{0.6}(x_{1})}^{-} = \frac{\sum_{x \in RS_{B} ^{0.6}(x_{1})} f(x, a_{1})^{-}}{|RS_{B} ^{0.6}(x_{1})|} = \frac{f(x_{1},a_{1})^{-} + f(x_{2},a_{1})^{-} + f(x_{3},a_{1})^{-} + f(x_{4},a_{1})^{-}}{4} = \frac{0 + 0.2 + 0.1 + 0}{4} = 0.075,\\
&\overline{V}_{RS_{a_{1}}^{0.6}(x_{1})}^{+} = \frac{\sum\limits_{x \in RS_{B} ^{0.6}(x_{1})} f(x_{1},a_{1})^{+}}{|RS_{B} ^{0.6}(x_{1})|} = \frac{f(x_{1},a_{1})^{+} + f(x_{2},a_{1})^{+} + f(x_{3}, a_{1})^{+} + f(x_{4},a_{1})^{+}}{4} = \frac{0.5 + 0.6 + 0.8 + 0.9}{4} = 0.7.
\end{align*}
Therefore, we have $\overline{V}_{RS_{a_{1}}^{0.6}(x_{1})} = [0.075, 0.7]$. In order to compute $CD_{RS_{B} ^{0.6}}$, we choose the similarity degree $SS$. Thus
\begin{equation*}
CD_{RS_{a_{1}}^{0.6}(x_{1})} = \frac{\sum_{x \in RS_{B} ^{0.6}(x_{1})}  SS(f(a_{1}, x), \overline{V_{RS_{B}^{0.6}(x_{1})}^{a_{1}}})}{|RS_{B} ^{0.6}(x_{1})|} = \frac{0.804 + 0.820 + 0.914 + 0.847}{4} = 0.846.
\end{equation*}
In the same way, we have $CD_{RS_{a_{2}}^{0.6}(x_{1})}=0.858$ and $CD_{RS_{a_{3}}^{0.6}(x_{1})} = 0.858$. Therefore,
$$CD_{RS_{B}^{0.6}(x_{1})}=\frac{1}{3} (0.846 + 0.858 + 0.858) = 0.854.$$
Similarly, we can obtain that $CD_{RS_{B}^{0.6}(x_{2})}= 0.849$, $CD_{RS_{B}^{0.6}(x_{3})}=0.711$, $CD_{RS_{B}^{0.6}(x_{4})}= 0.858$, and $CD_{RS_{B}^{0.6}(x_{5})}= 0.795$. It follows that $CD_{RS_{B} ^{0.6}}=\frac{1}{5} (0.854 + 0.849 + 0.711 + 0.858 + 0.795) = 0.813$.

Similarly, we can also compute the cluster degree of $RT_{B} ^{0.6}$, where we choose the similarity degree $ST$. We can obtain that $CD_{RT_{B} ^{0.6}}= 0.909$.

It is easy to see that $CD_{RS_{B}^{0.6}} \leq CD_{RT_{B}^{0.6}} \leq CD_{RF_{B}^{0.6}}$.
\end{ex}

In Example \ref{ex3}, we know that $CD_{RF_{B}^{0.6}}=1$. By Example \ref{ex1}, $RF_{B}^{0.6}$ is an equivalence relation. This coincides the result of Corollary \ref{cor1}.

\section{Comparison research on binary relations based on the transitive degree}\label{sec5}
In this section, we propose several methods to compare binary relations by means of the transitive degree. The purpose of these methods is to explore the difference between binary relations. Specially, when we obtain several binary relations from the same information system, which should relations be chosen to be used analyze this information system? In this section, we give reference for this problem.

In this section, we use the Face Recognition Dataset \cite{11,14} (See Table \ref{tb2}) to illustrate our methods. In Table \ref{tb2}, each interval value represents a measurement result from a face. This dataset consists of the $27$ observations and $6$ attributes $( i.e.~U = \{FRA1, FARA2, \cdot \cdot \cdot, ROM3\}, A = \{AD, BC, \cdot \cdot \cdot, GH\} )$. The six attributes consist of the length spanned by the eyes, the length between the eyes and etc..

\begin{table}[htbp]
\centering
\caption{\label{tb2} Face Recognition Dataset}
\small
\begin{tabular}{ccccccc}
\hline
subject & $AD$ & $BC$ & $AH$ & $DH$ & $EH$ & $GH$\\
\hline
FRA1 & $ [155.00,157.00]$ & $ [58.00,61.01]$ & $ [100.45,103.28]$ & $ [105.00,107.30]$ & $ [61.40,65.73]$ &$ [64.20,67.80]$\\
$FRA2$ & $ [154.00,160.01]$ & $ [57.00,64.00]$ & $ [101.98,105.55]$ & $ [104.35,107.30]$ & $ [60.88,63.03]$ & $[62.94,66.47]$\\
$FRA3$ & $[154.01,161.00]$ & $ [57.00,63.00]$ & $ [99.36,105.65]$ & $ [101.04,109.04]$ & $[60.95,65.60]$ & $ [60.42,66.40]$\\
$HUS1$ & $[168.86,172.84]$ & $[58.55,63.39]$ & $[102.83,106.53]$ & $[122.38,124.52]$ & $[56.73,61.07]$ & $ [60.44,64.54]$\\
$HUS2$ & $[169.85,175.03]$ & $[60.21,64.38]$ & $[102.94,108.71]$ & $[120.24,124.52]$ & $[56.73,62.37]$ & $ [60.44,66.84]$\\
$HUS3$ & $[168.76,175.15]$ & $[61.40,63.51]$ & $[104.35,107.45]$ & $[120.93,125.18]$ & $[57.20,61.72]$ & $ [58.14,67.08]$\\
$INC1$ & $[155.26,160.45]$ & $[53.15,60.21]$ & $[95.88,98.49]$ & $[91.68,94.37]$ & $[62.48,66.22]$ & $ [58.90,63.13]$\\
$INC2$ & $[156.26,161.31]$ & $[51.09,60.07]$ & $[95.77,99.36]$ & $[91.21,96.83]$ & $[54.92,64.20]$ & $[54.41,61.55]$\\
$INC3$ & $[154.47,160.31]$ & $[55.08,59.03]$ & $[93.54,98.98]$ & $[90.43,96.43]$ & $[59.03,65.86] $~&$[55.97,65.80]$\\
$ISA1$ & $[164.00,168.00]$ & $[55.01,60.03]$ & $[120.28,123.04]$ & $[117.52,121.02]$ & $[54.38,57.45]$ & $[50.80,53.25]$\\
$ISA2$ & $[163.00,170.00]$ & $[54.04,59.00]$ & $[118.80,123.04]$ & $[116.67,120.24]$ & $[55.47,58.67]$ & $[52.43,55.23]$\\
$ISA3$ & $[164.01,169.01]$ & $[55.00,59.01]$ & $ [117.38,123.11]$ & $[116.67,122.43]$ & $[52.80,58.31]$ & $[52.20,55.47]$\\
$JPL1$ & $[167.11,171.19]$ & $[61.03,65.01]$ & $[118.23,121.82]$ & $[108.30,111.20]$ & $[63.89,67.88]$ & $[57.28,60.83]$\\
$JPL2$ & $[169.14,173.18]$ & $[60.07,65.07]$ & $[118.85,120.88]$ & $[108.98,113.17]$ & $[62.63,69.07]$ & $[57.38,61.62]$\\
$JPL3$ & $[169.03,170.11]$ & $[59.01,65.01]$ & $[115.88,121.38]$ & $[110.34,112.49]$ & $[61.72,68.25]$ & $[59.46,62.94]$\\
$KHA1$ & $[149.34,155.54]$ & $[54.15,59.14]$ & $[111.95,115.75]$ & $[105.36,111.07]$ & $[54.20,58.14]$ & $[48.27,50.61]$\\
$KHA2$ & $[149.34,155.32]$ & $[52.04,58.22]$ & $[111.20,113.22]$ & $[105.36,111.07]$ & $[53.71,58.14]$ & $[49.41,52.80]$\\
$KHA3$ & $[150.33,157.26]$ & $[52.09,60.21]$ & $[109.04,112.70]$ & $[104.74,111.07]$ & $[55.47,60.03]$ & $[49.20,53.41]$\\
$LOT1$ & $[152.64,157.62]$ & $[51.35,56.22]$ & $[116.73,119.67]$ & $[114.62,117.41]$ & $[55.44,59.55]$ & $[53.01,56.60]$\\
$LOT2$ & $[154.64,157.62]$ & $[52.24,56.32]$ & $[117.52,119.67]$ & $[114.28,117.41]$ & $[57.63,60.61]$ & $[54.41,57.98]$\\
$LOT3$ & $[154.83,157.81]$ & $[50.36,55.23]$ & $[117.59,119.75]$ & $[114.04,116.83]$ & $[56.64,61.07]$ & $[55.23,57.80]$\\
$PHI1$ & $[163.08,167.07]$ & $[66.03,68.07]$ & $[115.26,119.60]$ & $[116.10,121.02]$ & $[60.96,65.30]$ & $[57.01,59.82]$\\
$PHI2$ & $[164.00,168.03]$ & $[65.03,68.12]$ & $[114.55,119.60]$ & $[115.26,120.97]$ & $[60.96,67.27]$ & $[55.32,61.52]$\\
$PHI3$ & $[161.01,167.00]$ & $[64.07,69.01]$ & $[116.67,118.79]$ & $[114.59,118.83]$ & $[61.52,68.68]$ & $[56.57,60.11]$\\
$ROM1$ & $[167.15,171.24]$ & $[64.07,68.07]$ & $[123.75,126.59]$ & $[122.92,126.37]$ & $[51.22,54.64]$ & $[49.65,53.71]$\\
$ROM2$ & $[168.15,172.14]$ & $[63.13,68.07]$ & $[122.33,127.29]$ & $[124.08,127.14]$ & $[50.22,57.14]$ & $[49.93,56.94]$\\
$ROM3$ & $[167.11,171.19]$ & $[63.13,68.03]$ & $[121.62,126.57]$ & $[122.58,127.78]$ & $[49.41,57.28]$ & $[50.99,60.46]$\\
\hline
\end{tabular}
\end{table}
\subsection{Comparison research on the relations from the viewpoint of the transitive degree}\label{sec51}
In Section \ref{sec23}, we know that one can establish binary relations, $RF_{B}^{\lambda}$, $RS_{B}^{\lambda}$ and $RT_{B}^{\lambda}$, from the same information system. In this section, we will analyze the difference among these binary relations from the viewpoint of the transitive degree. Firstly, we give an example in which we discuss the changes of transitive degrees of these binary relations with varying attribute subsets. Then based on the changes, we observe the difference among $RF_{B}^{\lambda}$, $RS_{B}^{\lambda}$ and $RT_{B}^{\lambda}$ so as to provide some scales for choosing relations in dealing with interval-valued information systems.

\begin{ex}\label{ex4}
In Table \ref{tb2}, we choose the following attribute subsets:
\begin{align*}
&B_{1} = \{AD\}, ~B_{2}=\{AD, BC\}, ~B_{3}=\{AD, BC, AH\}, ~B_{4}=\{AD, BC, AH, DH\}, \\
&~B_{5}=\{AD, BC, AH, DH, EH\} \text{ and}~B_{6}=\{AD, BC, AH, DH, EH, GH\}.
\end{align*}
Take $\lambda = 0.6$. By Eqs. (\ref{rf})-(\ref{rt}), we can obtain
$RF_{B_{1}}^{\lambda},~ RF_{B_{2}}^{\lambda}, \cdots, ~RF_{B_{6}}^{\lambda},$
$RS_{B_{1}}^{\lambda},~ RS_{B_{2}}^{\lambda}, \cdots, ~RS_{B_{6}}^{\lambda},$
$RT_{B_{1}}^{\lambda},~ RT_{B_{2}}^{\lambda}$, $\cdots, ~RT_{B_{6}}^{\lambda}$.
Next, by Definition \ref{trb}, we can compute their transitive degrees which are shown in Table \ref{tb3} and Figure \ref{fig1}. Table \ref{tb3} and Figure \ref{fig1} reflect the following facts:
\begin{itemize}
\item[$\cdot$]
Clearly, $B_{1}\subseteq B_{2}\subseteq\cdots\subseteq B_{6}$. However, $TD_{RS_{B_{1}}^{\lambda}}\leq TD_{RS_{B_{2}}^{\lambda}}\leq TD_{RS_{B_{3}}^{\lambda}}\leq TD_{RS_{B_{4}}^{\lambda}} \leq TD_{RS_{B_{5}}^{\lambda}}> TD_{RS_{B_{6}}^{\lambda}}$. That is to say, transitive degrees do not possess monotonicity.
\item[$\cdot$]
According to Theorem \ref{thm1}, we know that $RF_{B} ^{\lambda} \subseteq RT_{B} ^{\lambda}\subseteq RS_{B} ^{\lambda}$. By Table \ref{tb3} and Figure \ref{fig1}, we can see that $TD_{RF_{B_{2}} ^{\lambda}} > TD_{RT_{B_{2}} ^{\lambda}}> TD_{RS_{B_{2}} ^{\lambda}}$ and $TD_{RT_{B_{4}} ^{\lambda}} > TD_{RF_{B_{4}} ^{\lambda}}> TD_{RS_{B_{4}} ^{\lambda}}$. This illustrates that transitive degrees do not have the property of rank preservation for binary relations $RF_{B} ^{\lambda}$, $RT_{B} ^{\lambda}$ and $RS_{B}^{\lambda}$, and that the transitive degree is different from the coarse degree. In fact, we know that the transitive degree of a binary relation is not related to its coarse degree. Therefore the transitive degree that does not satisfy rank preservation is reasonable.
\item[$\cdot$]
From Figure \ref{fig1}, we can see that the transitive degree of $RF_{B}^{\lambda}$ is largest and changes smoothly. Thus $RF_{B}^{\lambda}$ is a good choice to address interval-valued information systems from the viewpoint of the transitive degree.
\end{itemize}

\begin{table}[htbp]
\centering
\caption{\label{tb3} Transitive degrees for different attribute subsets}
\begin{tabular}{ccccccc}
\hline
Transitive degrees & $B_{1}$ & $B_{2}$ & $B_{3}$ & $B_{4}$ & $B_{5}$ & $B_{6}$\\
\hline
$TD_{RF_{B}^{0.6}}$ & $0.677$ & $0.951$ & $0.951$ & $0.951$ & $1.000$ & $1.000$\\
$TD_{RS_{B}^{0.6}}$ & $0.620$ & $0.673$ & $0.673$ & $0.926$ & $0.975$ & $0.951$\\
$TD_{RT_{B}^{0.6}}$ & $0.623$ & $0.858$ & $0.858$ & $0.975$ & $0.877$ & $0.926$\\
\hline
\end{tabular}
\end{table}

\begin{figure}[ht]
\centering
\includegraphics[scale=0.6]{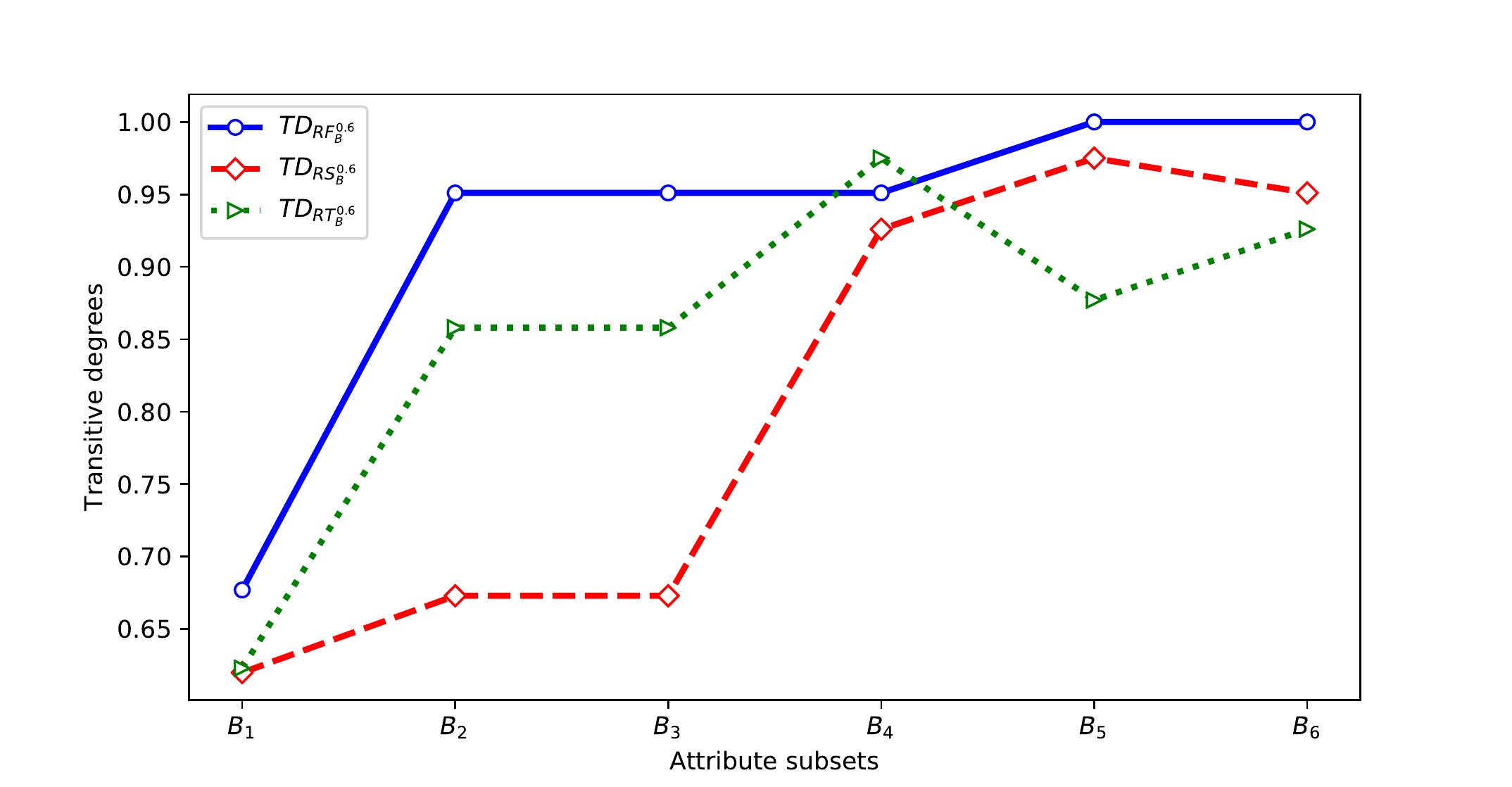}
\caption{The changes of transitive degrees for different attribute subsets}
\label{fig1}
\end{figure}
\end{ex}

In the following example, we discuss the changes of transitive degrees of $RF_{B}^{\lambda}$, $RS_{B}^{\lambda}$ and $RT_{B}^{\lambda}$ with varying the values of parameter.

\begin{ex}\label{ex5} In this example, we analyze the effects of parameters on transitive degree. In Table \ref{tb2}, we choose the attribute subset $B=\{AD, BC, AH, DH, EH, GH\}$. And we select $\lambda_{1}=0.1$, $\lambda_{2}=0.2$, ..., and $\lambda_{9} = 0.9$.

By Eqs. (\ref{rf}), (\ref{rs}) and (\ref{rt}), we can obtain
$RF_{B}^{\lambda_{1}},~ RF_{B}^{\lambda_{2}}, \cdots, ~RF_{B}^{\lambda_{6}},$
$RS_{B}^{\lambda_{1}},~ RS_{B}^{\lambda_{2}}, \cdots, ~RS_{B}^{\lambda_{6}},$
$RT_{B}^{\lambda_{1}},~ RT_{B}^{\lambda_{2}}$, $\cdots, ~RT_{B}^{\lambda_{6}}$.
Then, by Definition \ref{trb}, we can compute their transitive degrees, and we show them in Table \ref{tb4} and Figure \ref{fg2}. Table \ref{tb4} and Figure \ref{fg2} reflect the following facts:
\begin{itemize}
\item[$\cdot$]
From Figure \ref{fg2}, we can see that the transitive degrees of binary relations $RF_{B}^{\lambda}$ and $RT_{B}^{\lambda}$ change smoothly. Conversely, the transitive degree of $RS_{B}^{\lambda}$ is sensitive to the parameter $\lambda$. Thus $RF_{B}^{\lambda}$ and  $RT_{B}^{\lambda}$ are good choices to deal with interval-valued information systems from the viewpoint of transitive degree.
\end{itemize}

\begin{table}[htbp]
\centering
\caption{\label{tb4} Transitive degrees for different values of parameter}
\begin{tabular}{cccccccccc}
\hline
& $\lambda_{1}$ & $\lambda_{2}$ & $\lambda_{3}$ & $\lambda_{4}$ & $\lambda_{5}$ & $\lambda_{6}$ & $\lambda_{7}$ & $\lambda_{8}$ & $\lambda_{9}$\\
\hline
$T_{RF_{B}^{\lambda}}$ & $1.000$ & $0.951$ & $0.901$ & $0.951$ & $1.000$ & $1.000$ & $1.000$ & $1.000$ & $1.000$\\
$T_{RS_{B}^{\lambda}}$ & $0.992$ & $0.745$ & $0.358$ & $0.722$ & $0.926$ & $0.951$ & $0.951$ & $1.000$ & $1.000$\\
$T_{RT_{B}^{\lambda}}$ & $1.000$ & $1.000$ & $0.975$ & $0.975$ & $0.901$ & $0.926$ & $1.000$ & $1.000$ & $1.000$\\
\hline
\end{tabular}
\end{table}
\end{ex}

\begin{figure}[ht]
\centering
\includegraphics[scale=0.6]{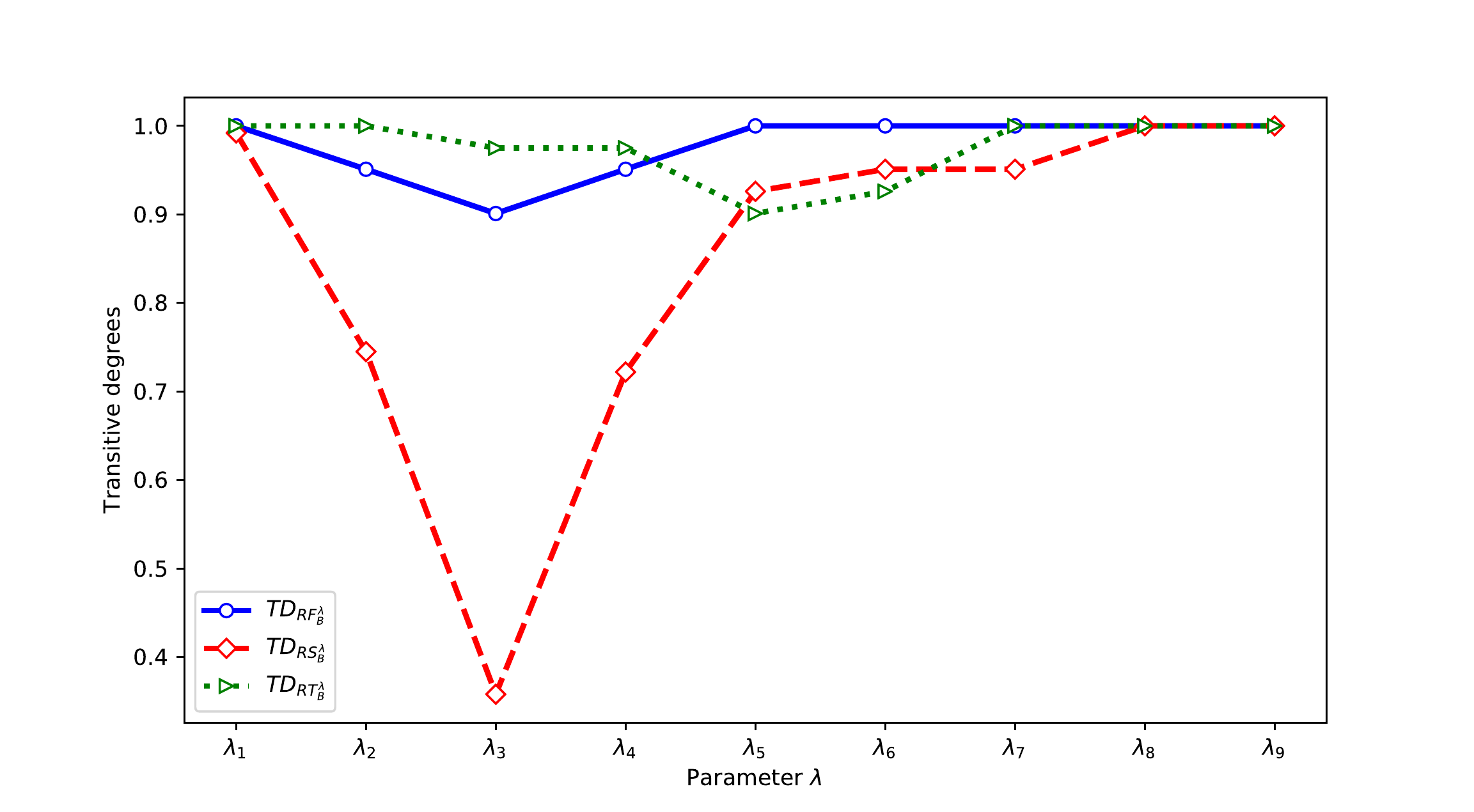}
\caption{The changes of transitive degrees for different values of parameter}
\label{fg2}
\end{figure}

\subsection{Comparison research on the relations from the viewpoint of the attribute reduction based on transitive degree}\label{sec52}
In this section, we want to demonstrate the effectiveness of the transitive degree, and analyze the rationality of $RF_{B} ^{\lambda}$, $RS_{B} ^{\lambda}$ and $RT_{B} ^{\lambda}$ from the perspective of attribute reduction.

By means of the transitive degree, we first give the definition of attribute reduction.
\begin{df}\label{trr}
Let $(U, A, V, f)$ be an interval-valued information system and $B \subseteq A$. $B$ is said a reduction of $A$ if $B$ satisfies the following conditions:

(1) $TD_{R_{B}} = TD_{R_{A}}$;

(2) $\forall a \in B$, $TD_{R_{B - \{a\}}} \neq TD_{R_{B }}$.
\end{df}

Significance measure of an attribute plays an important role in attribute reduction. It is generally defined by conditional entropy \cite{46}, dependency function, prior weighted entropy and posterior weighted entropy \cite{47}. In this section, we can define significance measure of an attribute by the transitive degree.

\begin{df}\label{sg1}
Let $(U, A, V, f)$ be an interval-valued information system and $B \subseteq A$. For $a \in A$, if $a\in B$, then the significance measure of the attribute $a$ with respect to $B$ is defined as follows:
\begin{equation}\label{sgm}
tsig^{inner} _{_{R_{B}}}(a) = |TD_{R_{B}} - TD_{R_{B - \{a\}}}|.
\end{equation}
If $a\notin B$, then the significance measure of the attribute $a$ with respect to $B$ is defined as follows:
\begin{equation}\label{sgm1}
tsig^{outer}_{_{R_{B}}}(a)=|TD_{R_{B \cup \{a\}}} - TD_{R_{B}}|.
\end{equation}
In particular, if $tsig^{inner}_{_{R_{B}}}(a)= 0$ holds, then $a$ is said reducible in $B$, and we can say that $a$ is a dispensable attribute.
\end{df}

According to Definitions \ref{trr} and \ref{sg1}, we can provide a heuristic algorithm of reduction proposed by Definition \ref{trr} as follows:

\text{\bf TDR Algorithm:}

$step~1$: $B \leftarrow \emptyset$.

$step~2$: For each $a \in A-B$, calculate each $tsig^{outer} _{_{R_{B}}}(a)$.

$step~3$: Choose $tsig^{outer} _{_{R_{B}}}(a^{\prime}) = \max \{tsig^{outer} _{_{R_{B}}}(a) \mid a \in A - B\}$ and $B \leftarrow B \cup \{a^{\prime}\}$.

$step~4$: If $TD_{R_{B}} = TD_{_{R_{A}}}$, then go to $step~5$. Otherwise, go to $step~2$;

$step~5$: For each $a \in B$, calculate $tsig^{inner} _{_{R_{B}}}(a)$. If $tsig^{inner}_{_{R_{B}}}(a) = 0$, then $B \leftarrow B - \{a\}$.

$step~6$: Output $B$. Now, $B$ is a reduction of $A$ by Definition \ref{trr}.

According to TDR Algorithm, we can give the following example so as to analyze the binary relations $RF_{B} ^{\lambda}$, $RS_{B} ^{\lambda}$ and $RT_{B} ^{\lambda}$.

\begin{ex}\label{ex6}
In this example, we consider the Face Recognition Dataset in Table \ref{tb2}. For convenience, we will denote the Face Recognition Dataset as $(U, A)$, and write $a_{1} = AD$, $a_{2} = BC$, $a_{3} = AH$, $a_{4} = DH$, $a_{5} = EH$ and $a_{6} = GH$. Take $\lambda = 0.6$. By Eqs. (\ref{fr})-(\ref{tr}), we can obtain three binary relations: $RF_{A}^{0.6}$, $RS_{A}^{0.6}$ and $RT_{A}^{0.6}$. By means of TDR Algorithm, we can acquire three group reductions corresponding to $RF_{A}^{0.6}$, $RS_{A}^{0.6}$ and $RT_{A}^{0.6}$. The results are shown in Table \ref{tb6}.
\end{ex}

\begin{table}[htbp]
\centering
\caption{\label{tb6} Reductions based on the transitive degree}
\begin{tabular}{ccc}
\hline
Binary relations &Reductions\\
\hline
$RF_{A}^{0.6}$ & $\{\{a_{1}, a_{6}\}, \{a_{2}, a_{6}\}, \{a_{3}, a_{6}\}, \{a_{4}, a_{6}\}\}$\\
$RS_{A}^{0.6}$ & $\{\{a_{4}, a_{5}, a_{6}\}\}$\\
$RT_{A}^{0.6}$ & $\{\{a_{2}, a_{4}\}, \{a_{3}, a_{4}\}\}$\\
\hline
\end{tabular}
\end{table}

By Table \ref{tb6}, we can view that the reductions based on $RF_{B}^{\lambda}$ are similar to that based on $RS_{B}^{\lambda}$, while the reductions based on $RT_{B}^{\lambda}$ have obvious difference with the former two cases. This illustrates that $RF_{B}^{\lambda}$ and $RS_{B}^{\lambda}$ have universality. Hence $RF_{B}^{\lambda}$ and $RS_{B}^{\lambda}$ are good choice to investigate interval-valued information systems from the view point of attribute reduction. In addition, In Figure \ref{fig1}, we can see that the relations $RF_{B}^{\lambda}$ and $RT_{B}^{\lambda}$ are good choice from the viewpoint of the transitive degree. In summary, $RF_{B}^{\lambda}$ is more suitable than the others when we use rough set approach to investigate the interval-valued information systems.

\section{Comparison research on binary relations based on the cluster degree}
In this section, we propose several methods to compare binary relations by means of the cluster degree. The purpose of these methods is to explore the classification ability of binary relations. We want to give reference for evaluating binary relations from the viewpoint of clustering.

\subsection{Comparison research on the relations from the viewpoint of the cluster degree}
In this paper, we know that three binary relations, $RF_{B} ^{\lambda}$, $RS_{B} ^{\lambda}$ and $RT_{B} ^{\lambda}$, can be established from the same information system. In this section, we will analyze the classification ability of these binary relations from the viewpoint of cluster degree.

In the following example, we discuss the changes of cluster degrees of these binary relations with varying attribute subsets. Then we observe the difference among $RF_{B}^{\lambda}$, $RS_{B}^{\lambda}$ and $RT_{B}^{\lambda}$.

\begin{ex}\label{ex8}
In Table \ref{tb2}, we choose the following attribute subsets:
\begin{align*}
&B_{1} = \{AD\}, ~B_{2}=\{AD, BC\}, ~B_{3}=\{AD, BC, AH\}, ~B_{4}=\{AD, BC, AH, DH\}, \\
&~B_{5}=\{AD, BC, AH, DH, EH\} \text{ and} ~B_{6}=\{AD, BC, AH, DH, EH, GH\}.
\end{align*}
Take $\lambda = 0.6$, by Eqs. (\ref{rf}), (\ref{rs}) and (\ref{rt}), we can obtain
$RF_{B_{1}}^{\lambda},~ RF_{B_{2}}^{\lambda}, \cdots, ~RF_{B_{6}}^{\lambda},$
$RS_{B_{1}}^{\lambda},~ RS_{B_{2}}^{\lambda}, \cdots, ~RS_{B_{6}}^{\lambda},$
$RT_{B_{1}}^{\lambda},~ RT_{B_{2}}^{\lambda}$, $\cdots, ~RT_{B_{6}}^{\lambda}$.
Next, by Definition \ref{jlr}, we can compute their cluster degrees, and we show them in Table \ref{tb8} and Figure \ref{fig2}. Table \ref{tb8} and Figure \ref{fig2} reflect the following facts:
\begin{itemize}
\item[$\cdot$]
From Figure \ref{fig2}, the cluster degree $CD_{RF_{B} ^{0.6}}$ is the biggest one. This illustrates that $RF_{B}^{0.6}$ has the best classification ability.
\item[$\cdot$]
From Figure \ref{fig2}, we can see that $CD_{RF_{B} ^{0.6}}$ is more sensitive than that of $CD_{RS_{B} ^{0.6}}$ and $CD_{RT_{B} ^{0.6}}$ about attribute subset changes. Thus we can use $CD_{RF_{B} ^{0.6}}$ to measure the significance of an attribute and to make attribute reduction.
\item[$\cdot$]
In terms of stability, the $CD_{RT_{B} ^{0.6}}$ is the most stable with attribute subset changes. That is to say, the classification ability of $RT_{B}^{0.6}$ has fewer impacts about attribute subset changes.
\end{itemize}
\end{ex}

In the following example, we discuss the changes of cluster degrees of $RF_{B}^{\lambda}$, $RS_{B}^{\lambda}$ and $RT_{B}^{\lambda}$ with varying the values of parameter.

\begin{table}[htbp]
\centering
\caption{\label{tb8} Cluster degrees for different attribute subsets}
\begin{tabular}{ccccccc}
\hline
& $B_{1}$ & $B_{2}$ & $B_{3}$ & $B_{4}$ & $B_{5}$ & $B_{6}$\\
\hline
$CD_{RF_{B} ^{0.6}}$ & $0.824$ & $0.917$ & $0.920$ & $0.946$ & $0.988$ & $1.000$\\
$CD_{RS_{B} ^{0.6}}$ & $0.802$ & $0.824$ & $0.822$ & $0.870$ & $0.880$ & $0.873$\\
$CD_{RT_{B} ^{0.6}}$ & $0.863$ & $0.899$ & $0.903$ & $0.924$ & $0.922$ & $0.944$\\
\hline
\end{tabular}
\end{table}

\begin{figure}[ht]
\centering
\includegraphics[scale=0.6]{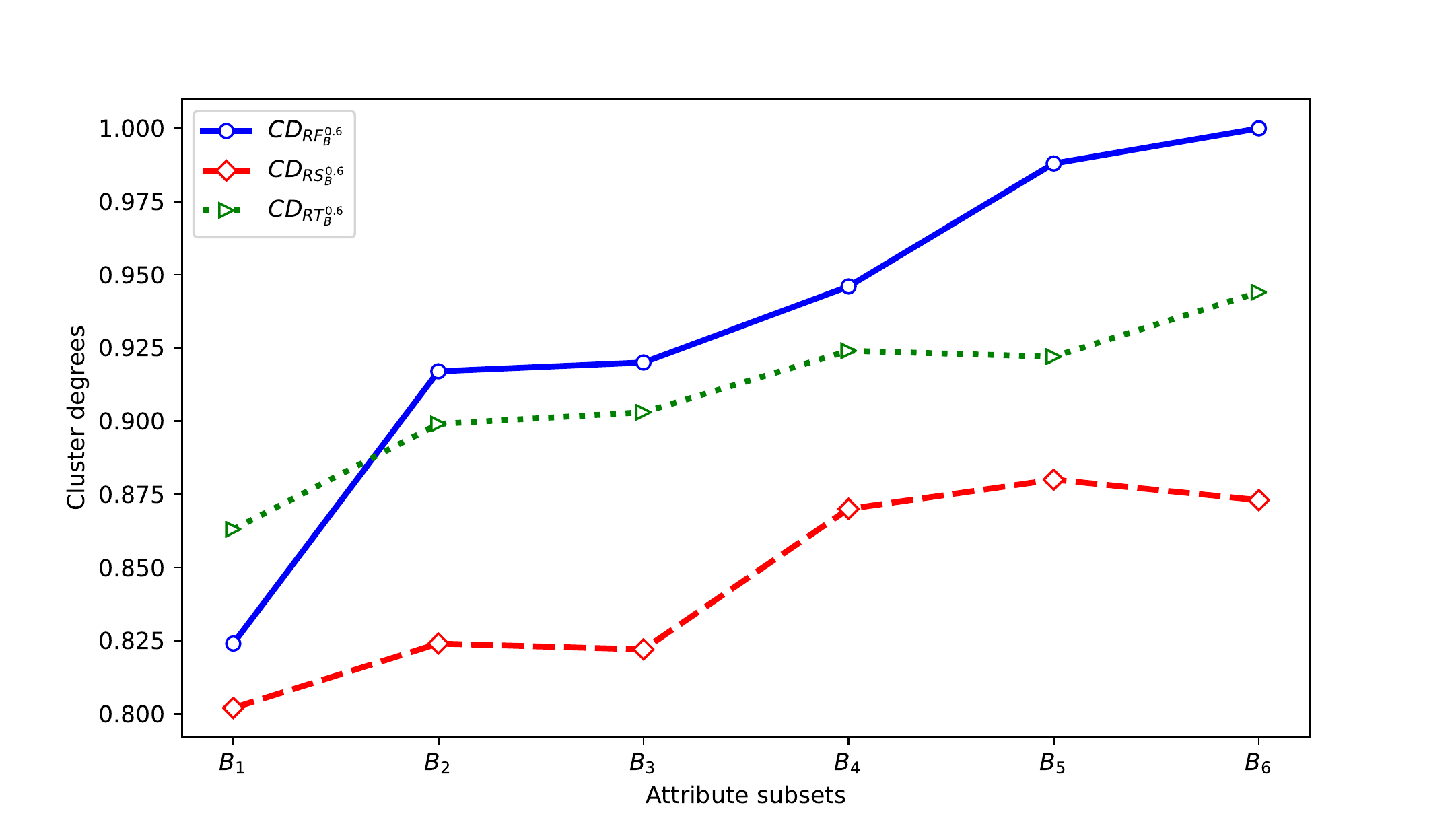}
\caption{The changes of cluster degrees for different attribute subsets}
\label{fig2}
\end{figure}

\begin{ex}\label{ex9}
In this example, we analyze the effects of parameters on the cluster degree. In Table \ref{tb2}, we choose the attribute subset $B=\{AD, BC, AH, DH, EH, GH\}$. And we select $\lambda_{1}=0.1$, $\lambda_{2}=0.2$, ..., and $\lambda_{9} = 0.9$.

By Eqs. (\ref{rf}), (\ref{rs}) and (\ref{rt}), we can obtain
$RF_{B}^{\lambda_{1}},~ RF_{B}^{\lambda_{2}}, \cdots, ~RF_{B}^{\lambda_{6}},$
$RS_{B}^{\lambda_{1}},~ RS_{B}^{\lambda_{2}}, \cdots, ~RS_{B}^{\lambda_{6}},$
$RT_{B}^{\lambda_{1}},~ RT_{B}^{\lambda_{2}}$, $\cdots, ~RT_{B}^{\lambda_{6}}$.
Then, by Definition \ref{jlr}, we can compute their cluster degrees, and we show them in Table \ref{tb9} and Figure \ref{fig3}. Table \ref{tb9} and Figure \ref{fig3} reflect the following facts:
\begin{itemize}
\item[$\cdot$]
As $\lambda$ becomes larger, $CD_{RF_{B} ^{\lambda}}$, $CD_{RS_{B} ^{\lambda}}$ and $CD_{RT_{B} ^{\lambda}}$ are monotonically increasing and reach the maximum value 1.
\item[$\cdot$]
For each $\lambda$, $CD_{RF_{B} ^{\lambda}}$ and $CD_{RT_{B} ^{\lambda}}$ are always more than $CD_{RS_{B} ^{\lambda}}$. That is to say, $CD_{RF_{B} ^{\lambda}}$ and $CD_{RT_{B} ^{\lambda}}$ possess better classification ability. Furthermore, $CD_{RF_{B} ^{\lambda}}$ is the first to achieve the maximum value 1. This illustrates that $CD_{RF_{B} ^{\lambda}}$ is more sensitive with $\lambda$ changes. Thus we can use $CD_{RF_{B} ^{\lambda}}$ to define the significance measure of an attribute so as to make attribute reduction.
\end{itemize}
\end{ex}

\begin{table}[htbp]
\centering
\caption{\label{tb9} Cluster degrees for different values of parameter}
\begin{tabular}{cccccccccc}
\hline
Cluster Degrees& $\lambda_{1}$ & $\lambda_{2}$ & $\lambda_{3}$ & $\lambda_{4}$ & $\lambda_{5}$ & $\lambda_{6}$ & $\lambda_{7}$ & $\lambda_{8}$ & $\lambda_{9}$\\
\hline
$CD_{RF_{B}^{\lambda}}$ & $0.740$ & $0.747$ & $0.799$ & $0.855$ & $0.964$ & $1.000$ & $1.000$ & $1.000$ & $1.000$\\
$CD_{RS_{B}^{\lambda}}$ & $0.540$ & $0.564$ & $0.658$ & $0.820$ & $0.851$ & $0.873$ & $0.928$ & $1.000$ & $1.000$ \\
$CD_{RT_{B}^{\lambda}}$ & $0.878$ & $0.878$ & $0.881$ & $0.899$ & $0.901$ & $0.944$ & $0.987$ & $0.997$ & $1.000$\\
\hline
\end{tabular}
\end{table}
\begin{figure}[ht]
\centering
\includegraphics[scale=0.6]{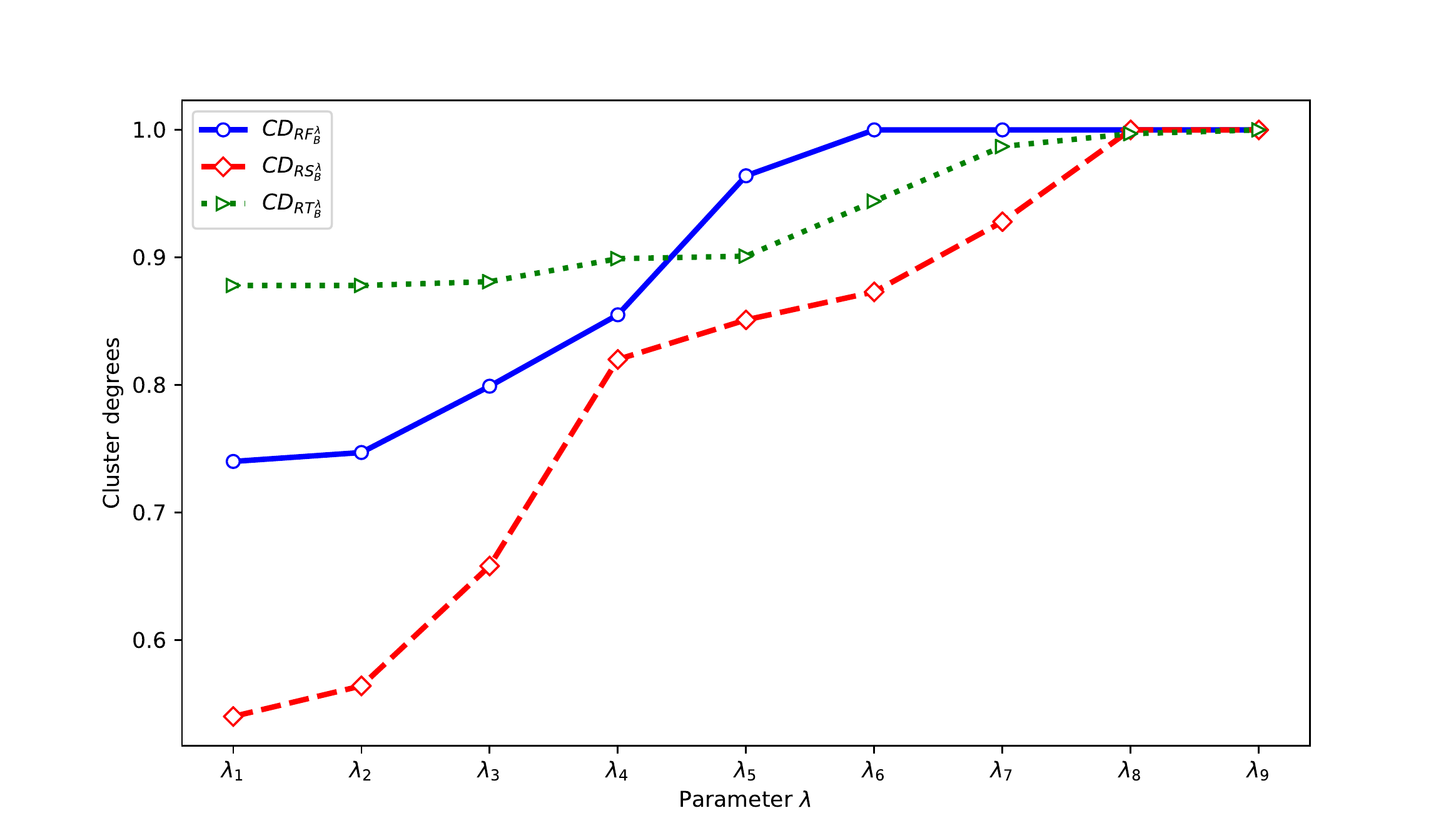}
\caption{The changes of cluster degrees for different values of parameter}
\label{fig3}
\end{figure}

\subsection{Comparison research on the relations from the viewpoint of the attribute reduction based on the cluster degree }
In this section, we analyze which similarity relation is the most rational for $RF_{B} ^{\lambda}$, $RS_{B} ^{\lambda}$ and $RT_{B} ^{\lambda}$ by means of attribute reduction based on the cluster degree.

By means of the cluster degree, we first give the definition of attribute reduction.
\begin{df}\label{cdr1}
Let $(U, A, V, f)$ be an interval-valued information system and $B \subseteq A$. $B$ is said a reduction of $A$ if $B$ satisfies the following conditions:

(1) $CD_{R_{B}} = CD_{R_{A}}$.

(2) $\forall a \in B$, $CD_{R_{B - \{a\}}} \neq CD_{R_{B }}$.
\end{df}

Similar to Section \ref{sec52}, in this section, we define significance measure of an attribute in terms of the cluster degree, and give the algorithm of attribute reduction.

\begin{df}\label{cdr2}
Let $(U, A, V, f)$ be an interval-valued information system and $B \subseteq A$. For $a \in A$, if $a\in B$, then the significance measure of the attribute $a$ with respect to $B$ is defined as follows:
\begin{equation}\label{sgm}
csig^{inner} _{_{R_{B}}}(a) = |CD_{R_{B}} - CD_{R_{B - \{a\}}}|.
\end{equation}
If $a\notin B$, then the significance measure of the attribute $a$ with respect to $B$ is defined as follows:
\begin{equation}\label{sgm1}
csig^{outer}_{_{R_{B}}}(a)=|CD_{R_{B \cup \{a\}}} - CD_{R_{B}}|.
\end{equation}
In particular, if $csig^{inner}_{_{R_{B}}}(a)= 0$, then $a$ is said reducible in $B$, and we can say that $a$ is a dispensable attribute.
\end{df}

According to Definitions \ref{cdr1} and \ref{cdr2}, we can provide a heuristic algorithm of reduction proposed by Definition \ref{cdr1} as follows:

\text{\bf CDR Algorithm:}

$step~1$: $B \leftarrow \emptyset$.

$step~2$: For each $a \in A-B$, calculate each $csig^{outer} _{_{R_{B}}}(a)$.

$step~3$: Choose $csig^{outer} _{_{R_{B}}}(a^{\prime}) = \max \{csig^{outer} _{_{R_{B}}}(a) \mid a \in A - B\}$ and $B \leftarrow B \cup \{a^{\prime}\}$.

$step~4$: If $CD_{R_{B}} = CD_{_{R_{A}}}$, then go to $step~5$. Otherwise, go to $step~2$;

$step~5$: For each $a \in B$, calculate $csig^{inner} _{_{R_{B}}}(a)$. If $csig^{inner}_{_{R_{B}}}(a) = 0$, then $B \leftarrow B - \{a\}$.

$step~6$: Output $B$. Now, $B$ is a reduction of $A$ by Definition \ref{cdr1}.

According to CDR Algorithm, we can give the following example so as to analyze the binary relations, $RF_{B} ^{\lambda}$, $RS_{B} ^{\lambda}$ and $RT_{B} ^{\lambda}$.

\begin{ex}\label{ex7}
In this example, we consider the Face Recognition Dataset in Table \ref{tb2}. For convenience, we will denote the Face recognition dataset as $(U, A)$, and write $a_{1} = AD$, $a_{2} = BC$, $a_{3} = AH$, $a_{4} = DH$, $a_{5} = EH$ and $a_{6} = GH$. Take $\lambda = 0.6$. By Eqs. (\ref{fr})-(\ref{tr}), we can obtain three binary relations: $RF_{A}^{0.6}$, $RS_{A}^{0.6}$ and $RT_{A}^{0.6}$. By means of CDR Algorithm, we can acquire three group reductions corresponding to $RF_{A}^{0.6}$, $RS_{A}^{0.6}$ and $RT_{A}^{0.6}$. The results are shown in Table \ref{tb11}.
\end{ex}

\begin{table}[htbp]
\centering
\caption{\label{tb11} Reductions based on the cluster degree}
\begin{tabular}{ccc}
\hline
Binary relations &Reductions\\
\hline
$RF_{B}^{0.6}$ & $\{\{a_{2}, a_{4}, a_{5}, a_{6}\}, \{a_{3}, a_{4}, a_{5}, a_{6}\}\}$\\
$RS_{B}^{0.6}$ & $\{\{a_{1}, a_{2}, a_{3}, a_{4}, a_{5}, a_{6}\}\}$\\
$RT_{B}^{0.6}$ & $\{\{a_{1}, a_{2}, a_{4}, a_{5} , a_{6}\}, \{a_{1}, a_{3}, a_{4}, a_{5} , a_{6}\}\}$\\
\hline
\end{tabular}
\end{table}

By Table \ref{tb11}, we have $\{a_{2}, a_{4}, a_{5}, a_{6}\}, \{a_{3}, a_{4}, a_{5}, a_{6}\}\subseteq\{a_{1}, a_{2}, a_{3}, a_{4}, a_{5}, a_{6}\}$, and then we conclude that the reductions based on $RF_{B}^{0.6}$ are subset of the reductions based on $RS_{B}^{0.6}$. In the same way, we can obtain that the reductions based on $RF_{B}^{0.6}$ are also subset of the reduction based on $RT_{B}^{0.6}$. This shows that we can choose the key attributes by using $RF_{B}^{0.6}$ to construct reduction algorithm.  Therefore $RF_{B}^{\lambda}$ is better choice than $RS_{B}^{0.6}$ and $RT_{B}^{\lambda}$ when we use rough set approach to investigate the interval-valued information systems. This coincides with the results analyzed by Section \ref{sec5}.

\section{Conclusions}
This paper presents two kinds of quantitative tools to measure binary relations, the transitive degree and the cluster degree. The transitive degree can represent the degree that a similarity relation is close to an equivalence relation. The cluster degree can be used to analyze the classification capability of binary relations induced by an interval-valued information system, and it can reflect the accuracy of the data from the interval-valued information system. By means of the transitive degree and the cluster degree, we construct some ways to distinguish binary relations. Examples illustrate that the two measures are very useful. In summary, we think that the measures are helpful to study information systems (not limit interval-valued information systems).\\

\noindent{\bf Declaration of interests}

The authors declare that they have no known competing financial interests or personal relationships that could have appeared to influence the work reported in this paper.








\end{document}